\documentclass{article}

\usepackage[final,nonatbib]{neurips_2020}

\usepackage{microtype}
\usepackage{graphicx}
\usepackage{subfigure}
\usepackage{booktabs} 

\usepackage[utf8]{inputenc} 
\usepackage[T1]{fontenc}    
\usepackage{url}            
\usepackage{booktabs}       
\usepackage{amsfonts}       
\usepackage{nicefrac}       
\usepackage{pgfplots}
\usepackage{pgfplotstable}
\usepackage{caption}
\usepackage{layouts}
\usepackage{xcolor}

\usepackage{amsmath}
\usepackage{latexsym}
\usepackage{xcolor}
\usepackage{xspace}
\usepackage{dsfont}
\usepackage{thmtools} 
\usepackage{thm-restate}
\usepackage{algorithm}
\usepackage{enumerate}
\usepackage{amsmath}
\usepackage{algorithmic}
\usepackage{mathtools}
\usepackage{wrapfig}
\usepackage{amsmath}
\usepackage{latexsym}
\usepackage{xcolor}
\usepackage{xspace}
\usepackage{dsfont}
\usepackage{thmtools} 
\usepackage{thm-restate}
\usepackage[capitalize, nameinlink]{cleveref}

\usepackage[capitalize, nameinlink]{cleveref}
\crefname{theorem}{Theorem}{Theorems}
\Crefname{lemma}{Lemma}{Lemmas}
\Crefname{claim}{Claim}{Claims}
\Crefname{observation}{Observation}{Observations}
\Crefname{invariant}{Invariant}{Invariants}

\newtheorem{theorem}{Theorem}[section]
\newtheorem{lemma}[theorem]{Lemma}

\newtheorem{corollary}[theorem]{Corollary}

\newtheorem{definition}[theorem]{Definition}

\newenvironment{proof}{\par\noindent\textit{Proof.}}{ \hfill $\Box$\par\bigskip\par}
\newenvironment{proofof}[1]{\par\noindent\textit{Proof of #1.}}{$\Box$\par\bigskip\par}

\newcommand{\reald}{\mathbb{R}^d}
\newcommand{\OPT}{\operatorname{OPT}}

\newcommand{\EE}{\operatorname{\mathbf{E}}}
\newcommand{\ee}[1]{\EE \left[ #1 \right]}

\newcommand{\dist}{\textsc{Dist}\xspace}
\newcommand{\MaxDist}{\textsc{MaxDist}\xspace}

\newcommand{\p}{\textbf{p}}

\newcommand{\cF}{\mathcal{F}}

\newcommand{\cH}{\mathcal{H}}

\newcommand{\kmpp}{\textsc{k-means++}\xspace}
\newcommand{\kmc}{\textsc{Afkmc2}\xspace}
\newcommand{\rejsam}{\textsc{RejectionSampling}\xspace}

\newcommand{\fastkm}{\textsc{Fastk-means++}\xspace}
\newcommand{\random}{\textsc{UniformSampling}\xspace}

\newcommand{\mtdist}{\textsc{MultiTreeDist}\xspace}
\newcommand{\mtsample}{\textsc{MultiTreeSample}\xspace}
\newcommand{\mtopen}{\textsc{MultiTreeOpen}\xspace}
\newcommand{\mt}{\textsc{MultiTree}\xspace}
\newcommand{\PhiLSH}{\ensuremath{\Phi_{\textrm{LSH}}}\xspace}
\newcommand{\tdist}{\textsc{TreeDist}\xspace}
\newcommand{\mtinit}{\textsc{MultiTreeInit}\xspace}

\def\fullversion{1}

\title{Fast and Accurate $k$-means++ via Rejection Sampling}
\author{
	Vincent Cohen-Addad\thanks{Equal contribution}  \\
	Google Research \\
	\texttt{cohenaddad@google.com} \\
	\And
	Silvio Lattanzi\footnotemark[1] \\
	Google Research\\
	\texttt{silviol@google.com} \\
	\AND
	Ashkan Norouzi-Fard\footnotemark[1] \\
	Google Research \\
	\texttt{ashkannorouzi@google.com} \\
	 \And
	Christian Sohler\footnotemark[1] \thanks{Work was partially done while author was visiting researcher at Google Research, Switzerland.}\\
	University of Cologne\\
	\texttt{csohler@uni-koeln.de} \\
	\And
	Ola Svensson\footnotemark[1] \\
	EPFL \\
	\texttt{ola.svensson@epfl.ch} \\
}
\begin{document}

\maketitle

\begin{abstract}
$k$-means++ \cite{arthur2007k} is a widely used clustering algorithm that is easy to implement, has nice
theoretical guarantees and strong empirical performance. Despite its wide adoption, $k$-means++ sometimes suffers from being slow  
on large data-sets so a natural question has been to obtain more  efficient algorithms with similar guarantees. In this paper,
we present  a near linear time algorithm for $k$-means++ seeding. Interestingly our algorithm obtains the same 
theoretical guarantees as $k$-means++ and significantly improves earlier results on fast $k$-means++ seeding. 
Moreover, we show empirically that our algorithm is significantly faster than $k$-means++ and
obtains solutions of equivalent quality.
\end{abstract}

\section{Introduction}
Clustering is a fundamental task in machine learning with many applications in data analysis and machine learning. One particularly
important variant is $k$-means clustering: Given a set of $n$ points in $\mathbb{R}^d$ the goal is to find a partition of the points into $k$ clusters such that the sum of squared distance to the cluster centers (which are the means of the clusters) is minimized. 

A popular method to compute a good clustering with respect to the 
$k$-means objective is the \kmpp algorithm \cite{arthur2007k}. The algorithm is a combination of
a randomized procedure for finding a set of $k$ starting centers (often referred to as the \emph{seeding}) with the classic local improvement algorithm by Lloyd \cite{L82}.
The seeding step samples the first center uniformly at random. In the remaining 
iterations the algorithm samples a center from the $D^2$-distribution, where the
probability of sampling a point is proportional to the squared distance to the current set
of centers.

The \kmpp algorithm is easy to implement, has strong theoretical guarantees (an $O(\log k)$ approximation guarantee), and performs empirically well. However, the running time of $\Theta(dnk)$\footnote{Even assuming a constant number of Lloyd's algorithm steps.} becomes impractical for many very large datasets. Therefore, a lot of previous work focused on speeding up the \kmpp seeding \cite{NIPS2016_6478,BLHK16} as well as Lloyd's algorithm \cite{ding2015yinyang,bottesch2016speeding,newling2016fast,curtin2017dual}. 

To obtain a fast seeding algorithm, Bachem et al. \cite{NIPS2016_6478} and \cite{BLHK16} use an MCMC algorithm to generate a set of $k$ centers that follows the \kmpp distribution. They provide different versions of their algorithm that provide trade-offs between theoretical guarantees and empirical running time. Interestingly, under certain assumptions on
the inputs and when $k$ is small, their algorithm may even run in sublinear time in the input size.
However, all versions of their algorithms have a running time of $\Omega(k^2)$\footnote{We note that
a similar running time can be achieved also via coresets~\cite{har2004coresets, chen2009coresets} but it is challenging to go below the $\Omega(k^2)$ barrier.} and so it does not scale for massive datasets and moderate values of $k$ (i.e. $500-1000$). Another important drawback from their results is that the guarantees on the quality of the solution output by their algorithms are  weaker than the original \kmpp guarantee (since their approximation is additive in the worst case). 

\paragraph{Our contribution}  In this paper we present a new algorithm that provably achieves near-linear running time while offering similar guarantees as the original \kmpp algorithm. In particular: 
\begin{itemize}
    \item We introduce a new seeding algorithm that for constant $\varepsilon>0$ has near-linear running time $\widetilde{O}\left(nd +  \left(n\log(\Delta)\right)^{1+\varepsilon}\right)$ and returns a $O_\varepsilon(\log k)$ approximate solution, where $n$ is the number of points in the dataset and $\Delta$
    is the ratio between the maximum distance and the minimum distance between two points in the dataset, see Corollary~\ref{cor:main}. Our algorithm also has  the advantage that, in the stated running time, it computes the solution for \emph{all} values of $k=1,2,\ldots, n$. 
    \item  We compare 
    the performances of our seeding technique with the baselines \kmpp 
    and \kmc from~\cite{NIPS2016_6478} on various classic datasets. Our algorithms outperform the baselines even for moderate 
    values of $k$ (e.g.: $k= 1000$) and the speed-up is an order of magnitude for larger values of $k$\footnote{While the large $k$ setting is not the most studied setting,
    it still has many practical applications. For instance in  spam and abuse~\cite{qian2010case, sheikhalishahi2015fast}, 
    near-duplicate detection~\cite{DBLP:journals/pvldb/HassanzadehCML09}, compression or reconciliation 
    tasks~\cite{DBLP:conf/edbt/RostamiSPR18}. Furthermore the large $k$ case is very interesting from a
    theoretical perspective and it gained attention in recent years~\cite{bhaskara2018distributed}.}.
   In addition, our algorithms output solutions of similar costs as
    \kmpp (as our theoretical results predict).
\end{itemize}

The main idea behind our method is to use an embedding into a collection of
trees to approximate the distances between the input points, and then
leverage the tree structure to speed-up the $D^2$-sampling of \kmpp.
To ensure that our sampling, which uses the approximate tree distances, leads
to a solution that is competitive with respect to \kmpp on the \emph{original} data, we "emulate" the $D^2$-distribution on the original data by additionally using rejection sampling. 
More concretely,\\
-- We first develop a new seeding algorithm \fastkm that computes
    in $\widetilde{O}(nd)$ time a solution. The near linear running time is obtained by first approximating the
    squared Euclidean distance using a multi-tree embedding and then by showing that
    one can efficiently perform $D^2$-sampling with respect to multi-tree distances.\\
-- We then argue that one can use  our sampling technique on  multi-tree distances in combination with rejection sampling so as to reproduce the 
    same distribution as used by \kmpp on the original distances and so to achieve 
    the same approximation guarantees that \kmpp. To ensure a fast running time, we calculate the rejection probability by using 
 locality-sensitive hashing (LSH) to approximately determine the nearest neighbor w.r.t. the original distances. \\
-- Finally, we show that our LSH based rejection sampling algorithm computes a solution with the same 
    expected approximation guarantee of $O(\log k)$ as the basic $\kmpp$ algorithm.



\section{Preliminaries}  

\paragraph{Basic notation.} We denote by $P \subseteq \reald$ the set of $n$ input points in a $d$ dimensional space and let $\Delta$
    be the ratio between the maximum distance and the minimum distance between two point in the dataset. 
The Euclidean distance between two points $x,y \in \mathbb{R}^d$ is denoted by $\dist(x,y) = ||x-y||_2$. We also let $\dist(x,C)=\min_{y\in C} \dist(x,y)$ be the distance of $x$ to the closest point in $C$. The goal in the $k$-means problem is to choose a set of $k$ centers $C \subseteq \mathbb{R}^d$ minimizing
$
\sum_{x \in P}  \dist(x,C)^2.
$


\paragraph{Tree embeddings.} Tree embedding is a well-known technique used in many different clustering problems  (see for example \cite{B96}). 
\iffalse
Intuitively, in a tree embedding one assigns each point in $P$ to a leaf of a tree and weight of the edges (starting from $1$) increases by a factor two going toward the root. The distance between two points is defined as the shortest path between them.
We explain this procedure for the sake of completeness in 
\if\fullversion1
\cref{app:tree-embedding}.
\else
the full version.
\fi
\else
We now explain a  simple version that is similar to \cite{I04} that will be used in our algorithm. We first compute an  upper bound $\MaxDist$ on the maximum distance between two points within a factor of $2$.\footnote{This can be done in time $O(nd)$, by selecting any point and by computing the maximum distance between that point and any other point in the dataset. Then multiply this distance by 2.} Second, we  add a random shift $0 \leq s\leq \MaxDist$ to each coordinate of all input points\footnote{Notice that this does not effect the distance between any two points and therefore the cost of any solution.}. Let $x\in P$ be any point in the data set. The root of the tree (at height zero) represents an axis-aligned cube of side length $2 \MaxDist$ centered at $x$. By selection, note that all the input points are inside this cube and we say that they \emph{belong to} this node of the tree.
We then partition this cube into $2^d$ axis-aligned subcubes of side length $\MaxDist$ and assign each point to the one that contains its coordinates. For each of these subcubes that contains a point, we create a node and add it as a child of the root in the tree (so their height is one), with edge weight  $\sqrt{d}\MaxDist$, i.e., the side length of the (parent) cube times $\sqrt{d}/2$. Notice that this is equal to the half of the maximum distance between any two coordinates in the parent cube. Also observe that the number of nodes at height one is at most $n$, since each node contains at least one point. We let the height of these edges  be zero. We then repeat this operation on the nodes until \emph{every} cube contains at most a single point. This results in a tree where all leaves are at the same height, the height is at most $H = O(\log(d\Delta))$, and there are at most $n$ nodes in each layer. Moreover, the running time of constructing each layer is $O(nd)$ since for each point we can determine in which subcube it belongs by going over its dimensions. The total running time is thus $O(nd\log(d\Delta))$. The distance between two points $p, q$ in the tree, denoted by $\tdist(p,q)$, is the length of the shortest path between $p$ and $q$ in the tree, or equivalently twice the length from one of them to their lowest common ancestor. 

\fi

\section{Multi-tree Embedding}
\label{sec:multitree}
Tree embedding is a powerful tool for designing approximation algorithms but it cannot be applied directly to the $k$-means problem. In fact there are simple examples that show that the expected distortion between  the squared distances of an $\ell_2$ metric and the \tdist is $\Omega(n)$. To overcome this limitation, we use \emph{three} tree embeddings with different random shifts and we define the distance between two points (denoted by \mtdist) to be the minimum \tdist among the distances in the three trees.
Interestingly, we show that this suffices to get a significantly better upper bound on the distortion. We refer to this simple procedure as $\mtinit()$. We note that the running time of $\mtinit()$ is asymptotically equal to that of a single tree embedding   $O(n d \log(d\Delta))$ since it initializes three tree embeddings. 

To analyze the expected distortion, define for any set $S \subseteq P$ and point $p \in P$, $\mtdist(p,S) = \min_{q \in S} \mtdist(p,q).$ 
The proof of the following bounds is provided in 
\if\fullversion1
\cref{app:multitreeperf}.
\else
the full version.
\fi
\begin{restatable}{lemma}{lemmamultitree} \label{lemma:approxdist}
  For any point $p$, and set $S$, we have 
  $
     \dist(p, S)^2 \le \mtdist(p,S)^2 \quad \text{ and } \quad
  \EE[\mtdist(p,S)^2] \leq O(d^2 \cdot \dist(p, S)^2)\, .
 $
\end{restatable}

\section{\fastkm Algorithm} \label{sec:fastkmeans}
Recall that the classic  \kmpp algorithm  samples the first center uniformly at random and in the remaining 
iterations $k$-means++  samples a center from the $D^2$-distribution, where the
probability of each point is proportional to the squared distance to its nearest current center. The most 
expensive operation in this procedure is to update the $D^2$-distribution after each sample. In fact,
the probability for a point to be selected may change in every round of the algorithm leading to $n$ updates in each of the $k$ iterations.

Our key idea here is to use the special structure of the  the multi-tree embedding to update the $D^2$-distribution with respect to those distances efficiently. This is intuitively possible since in the multi-tree metric every node can change its distance from the current set
of centers at most $O(\log(d\Delta))$ times. This is true because in order to decrease the distance between
a point $x$ and the set of centers in a single tree embedding,  the lowest common ancestor between $x$ and the closest center
has to get closer to $x$. The number of times that this can happen is bounded by the height of the tree.  Therefore, since the multi-tree embedding consists of three trees of height $O(\log(d\Delta))$, we have that the number of times a point can change its multi-tree distance to the set of opened centers is at most $O(\log(d\Delta))$.

\paragraph{\mtopen and \mtsample.}
To describe our algorithm we start by defining the procedures to update the distribution, 
\mtopen, and to compute a sample \mtsample. To achieve an efficient running time, both procedures act on a common data structure which consists of the following:
\begin{itemize}
    \item A weight $w_x$ for each point $x\in P$.
    \item A node-weighted balanced binary tree with a leaf for each of the $n$ points in $P$. We refer to this tree as the sample-tree so as to not confuse it with the trees in the multi-tree embedding. 
    \item For each node in each of  the trees in the multi-tree embedding, a bit saying whether this node is marked.
\end{itemize} 

For notational convenience, let $\mtdist(x, \emptyset)^2 = M$ for any point $x\in P$, where $M=16d\cdot \MaxDist^2$ is chosen to be an upper bound of  $\mtdist(p,q)^2$ for any two points $p$ and $q$.  
 If we let $S$ be the set of opened points (using calls to \mtopen), the data structure will satisfy the following invariants:
\begin{enumerate}
    \item For every $x\in P$, $w_x = \mtdist(x,S)^2$.
    \item The weight of each node in the sample-tree equals the sum of the weights of the points corresponding to the leaves in its subtree.
    \item A node $v$ in a tree $T$ in the multi-tree embedding is marked if there is a point in its subtree that has been opened, i.e., is in $S$; otherwise it is in unmarked. 
\end{enumerate}

So the data structure is initialized (when $S = \emptyset$) by  setting all weights $(w_x)_{x\in P}$ to $M$; setting the weight of each node in the sample-tree to $M$ times the number of points in its subtree; and by letting all nodes in the trees of the multi-tree embedding to be unmarked. In addition, for each tree $T$ in the multi-tree embedding and  for each node $v$ in $T$, we compute the set $P_T(v) \subseteq P$ of points in its subtree.  Note that the initialization of the weights and the sample-tree run in time $O(n)$ whereas the initialization of the unmarked notes and the sets $P_T(v)$ can be computed in time $O(n \log(d\Delta))$ by traversing the trees in the multi-tree embedding of height $O(\log(d\Delta))$. The total runtime of the initialization is thus $O(n \log(d\Delta))$.

We proceed to describe the procedure $\mtopen$ that opens a new point $x$ and updates the data structure to satisfy the invariants. We then describe the simpler procedure $\mtsample$ which samples a point $x$ with probability $w_x/(\sum_{y\in P} w(y))$, i.e., from the $D^2$-distribution with respect to the multi-tree distances. 

\begin{wrapfigure}[20]{R}{0.5\textwidth}
\begin{minipage}{0.5\textwidth}
\vspace{-3.1em}
\begin{algorithm}[H]
  \caption{\mtopen \label{alg:mtopen}}
  \vspace{0.1cm}
  \begin{algorithmic}[1]
    \REQUIRE A point $x\in P$ 
    \FOR {each tree $T$ in the multi-tree embedding}
    \STATE Let $v_0$ be the leaf of $T$ that $x$ belongs to. \label{mtopen:step:1}
    \STATE Traverse the tree towards the root forming a path $v_0, v_1, \ldots, v_\ell$ until either $v_\ell$ is the root or the parent of $v_\ell$ is marked.\label{mtopen:step:2}
    \STATE Mark $v_0, \ldots, v_\ell$. \label{mtopen:mark} 
    \FOR {each point $y$ in $P_T(v_\ell)$} \label{mtopen:for}
	    \IF {$\tdist_T(y,x)^2 < w_y$}
      	  \STATE $w_y \leftarrow \tdist^2_T(y,x)$
      	  \STATE Traverse the sample-tree from the leaf corresponding to $y$ to the root to update the node-weights   that depend on $w_y$.  \label{mtopen:sampletree}
      	  \ENDIF
      	  \ENDFOR
      	  \ENDFOR
  \end{algorithmic}
\end{algorithm} 
\end{minipage}
\end{wrapfigure}
The description of $\mtopen$ is given in \cref{alg:mtopen}. When the tree embedding is not clear from the context,  we use the notation $\tdist_T$ to denote the distances given by the tree embedding $T$. We now verify the invariants and give some intuition of the procedure. 
Let $S$ be the set of opened centers prior to this call to $\mtopen(x)$ and let $T$ be a tree  in the multi-tree embedding. When considering $T$, \mtopen starts in the leaf $v_0$ of $T$ that $x$ belongs to. It then traverses the tree towards the root forming a path $v_0, v_1, \ldots, v_\ell$ of nodes so that $v_\ell$ is either the root or its parent is already marked. 
The subtrees of these vertices are exactly those that contain $x$ but no other point in $S$, and so  Step~$4$ guarantees the third invariant. Now a key observation is that $\tdist_T(y, S \cup \{x\}) < \tdist_T(y, S)$ for exactly those points $y$ in $P_T(v_\ell)$. This holds because in order to decrease the distance between a point $y$ and the set of centers, with respect to the tree embedding $T$, the lowest common ancestor in $T$ between y and the closest center must get closer.    \mtopen considers each of these points and updates $w_y$ if $\tdist_T(y, x) < w_y$.  Since the procedure considers all three trees in the multi-tree embedding this guarantees the first invariant, i.e.,  that $w_y = \mtdist(y, S \cup \{x\})^2$  for every $y\in P$ at the end of the procedure.  The second invariant is guaranteed by Step~8 which updates all the nodes in the sample-tree so as to satisfy that invariant.  \mtopen therefore updates the data structure to satisfy the invariants. As the distance from a point $x$ to the centers is updated $O(\log(d\Delta))$ times and each time  the sample-tree is updated in time $O(\log n)$ (its height), we have the following running time (see 
\if\fullversion1
\cref{app:fastkmeansperf} for a formal argument).
\else
the full version for a formal argument).
\fi

\begin{restatable}{lemma}{treeopenrunningtime} \label{lemma:mtopen-fast-runtime}
The running time of opening any set $S$ of $k$ points (using calls to \mtopen) is  $O(n \log(d\Delta) \log n)$.
\end{restatable}

Having described how to open a new center, we proceed to describe the simpler algorithm for generating a sample. The pseudo-code of  \mtsample is given in~\cref{alg:mtsample}. \mtsample traverses  the sample-tree from the root to a random leaf by, at each intermediate node, randomly choosing one of its two children proportional to its weight. As the weight of each node in the sample-tree, equals the sum of weights of the points in its subtree (by the second invariant), this guarantees that a point $x$ is sampled with probability  $w_x/ \sum_{y\in P} w_y$, i.e., proportional to its weight. By the first invariant, this corresponds to sampling from the $D^2$-distribution with respect to the multi-tree distances\footnote{We remark that the idea of sampling in 
this way from a tree has been used in the context of constructing a coreset in \cite{AMRSLS12} 
(however, their tree depends on a partition of the data and is not necessarily balanced).}. Furthermore, the running time  of $\mtsample$ is $O(\log n)$ since the height of the sample-tree is $O(\log n)$. (Recall that the sample-tree is a balanced binary tree with $n$ leafs and is thus of height $O(\log n)$. Recall also that the sample-tree is a different tree from the tree embeddings)  We summarize these properties of $\mtsample$ in the following lemma.
\begin{minipage}{\textwidth}
\begin{minipage}{0.55\textwidth}
\begin{algorithm}[H]
  \caption{\mtsample}
  \label{alg:mtsample}
  \vspace{0.1cm}
  \begin{algorithmic}[1]
    \STATE Let $v$ be the root of the sample-tree.
    \WHILE {$v$ is not a leaf} 
    	\STATE Let $w(L)$ and $w(R)$ be the weight of its left and right child, respectively.
    	\STATE Update $v$ to be its left child with probability $\frac{w(L)}{w(L) + w(R)}$ and to be its right child with remaining probability $\frac{w(R)}{w(L) + w(R)}$.
    \ENDWHILE
    \ENSURE the point $x$ corresponding to the leaf $v$.
  \end{algorithmic}
\end{algorithm} 
\end{minipage}
\begin{minipage}{0.45\textwidth}
\begin{algorithm}[H]
  \caption{\fastkm}
  \label{alg:fastkm}
  \vspace{0.1cm}
  \begin{algorithmic}[1]
    \REQUIRE Set of points $P$, number of centers $k$. 
    \STATE Set $S \leftarrow \emptyset$ 
    \STATE \mtinit() 
    \WHILE {$|S| < k$} 
    	\STATE $x \leftarrow \mtsample()$ 
	  \STATE $S \leftarrow S \cup x$ 
      	  \STATE $\mtopen(x)$ 
    \ENDWHILE
    \ENSURE $S$
  \end{algorithmic}
\end{algorithm} 
\end{minipage}
\end{minipage}
\begin{lemma} \label[lemma]{lemma:mtsample-fast-runtime}
Let $S$ be the set of opened centers (using calls to $\mtopen$). Then $\mtsample$ runs in time $O(\log n)$ and each point $x\in P$ is output  with probability 
$
    \frac{\mtdist(x,S)^2}{\sum_{y\in P} \mtdist(y,S)^2}\,.
$
\end{lemma}

\paragraph{\fastkm.} We can now present a fast algorithm for the $k$-means problem (see \cref{alg:fastkm}) that samples each center from the $D^2$-distribution with respect to the distances given by the  multi-tree embedding. In the next section we show how to adapt the procedure so as to sample from the original $D^2$-distribution by using rejection sampling. 
The running time directly follows from that, the time to initialize the multi-tree embedding is $O(n d \log(d\Delta))$, the time to initialize the data structure used by \mtopen and \mtsample is $O(n \log(d\Delta))$, the total running time of $\mtopen$ is $O(n \log(d\Delta) \log n)$ (\cref{lemma:mtopen-fast-runtime}) and the running time of each call to \mtsample is $O(\log n)$ (\cref{lemma:mtsample-fast-runtime}).
\begin{corollary}
The running time of \fastkm is ${O}(nd\log(d\Delta) + n \log(d\Delta) \log n)$.
\end{corollary}
\section{Rejection Sampling Algorithm} \label{sec:rej-samp}

In this section we present an algorithm, \rejsam, that efficiently samples arbitrarily close to the $D^2$-distribution in the original metric. 
The algorithm is rather simple and its pseudo-code is given in \cref{alg:rejsam}. 
The main idea is to use the multi-tree embedding to sample candidate centers but then adjust the sampling probability
using rejection sampling.

As for \kmpp, the first center that we pick is chosen uniformly at random among all the points. 
For the rest of the $k-1$ centers, the idea is to sample a point $x$ using $\mtsample$, i.e., form the $D^2$-distribution with respect to the multi-tree distances. 
Then we open $x$ as a new center with probability proportional to its actual distance to the set of centers in the original metric over the distance in the multi-tree embedding. 
We repeat this procedure until we pick the rest of the $k-1$ centers. 
Interestingly, this rejection-sampling procedure guarantees that we sample each of the centers according to the actual $D^2$-distribution. 
However, the running time of this procedure is of $\Omega(k^2)$ since, for each point $x$ that we sample from the multi-tree, we have to find the closest open center which takes time $\Omega(k)$. In order to improve this running time, we use an approximate nearest neighbor data structure to approximate the distance between $x$ and the closest open center. 
This enables us to improve the running time to be near linear. The data structure that we use is based on the locality-sensitive hash (LSH) functions  developed for Euclidean metrics~\cite{andoni2006near}. 
 We only need to slightly modify their data structure to guarantee  monotonicity  as we explain in 
\if\fullversion1
\cref{sec:LSH-data-structure}.
\else
the full version.
\fi
\begin{theorem}[LSH data structure]
    For any set $P$ of $n$ points in $R^d$ and
    any parameter $c > 1$, there exists a data structure with operations Insert and Query that,  with probability at least $1-1/n$, have the following guarantees:
        (i) Insert($p$): \emph{Inserts} point $p\in P$ to the data structure    in time 
       ${O}\left(d \log(\Delta) \cdot \left(n\log(\Delta)\right)^{O(1/c^2)}\right)$. 
        (ii) Query($p$): Returns a point $q$ that has been inserted into the data structure that is at distance at most $c \cdot \delta$ from $p$, where $\delta$ is the minimum distance from $p$ to a point inserted to the data structure. The query
        time is ${O}\left(d\log(\Delta)\cdot \left(n\log(\Delta)\right)^{O(1/c^2)}\right)$.
        \\
     Furthermore, the data structure  is \emph{monotone under insertions}: the distance between $p$ and  Query($p$) is non-increasing after inserting more points.
     \label{thm:data_structure}
\end{theorem}

We say that the data structure is successful if the above guarantees hold. By the theorem statement, we know that the data structure is successful with probability at least $1-1/n$. The small failure probability will not impact the expected cost of our solution\footnote{To be completely formal: if we repeat our algorithm for $\log_n(4n\Delta^2 )$ times, then we know that with probability at least $1-1/(4n\Delta^2 )$ one of the runs is with a successful data structure. As squared-distances are at most $\MaxDist^2$ and at least $\MaxDist^2/(2\Delta)^2$, the total cost of a solution with a single opened center is at most $n\cdot \MaxDist^2$. Therefore,  the small failure probability of $1/(4n\Delta^2 )$ will not have a measurable impact on the expected cost of the best found clustering.}. 
We therefore  \emph{assume throughout the analysis that our data structure is successful}. In \cref{alg:rejsam} we present the pseudocode for our algorithm.

\begin{wrapfigure}[16]{R}{0.5\textwidth}
\begin{minipage}{0.5\textwidth}
\vspace{-2.1em}
\begin{algorithm}[H]
  \caption{\rejsam}
  \label{alg:rejsam}
  \vspace{0.1cm}
  \begin{algorithmic}[1]
    \REQUIRE Set of points $P$, number of centers $k$
    \STATE Set $S \leftarrow \emptyset$ 
    \STATE \mtinit()
    \WHILE {$|S| < k$}
    	\STATE $x \leftarrow \mtsample()$
    	\STATE {\bf{With probability}} $\min\{1,\frac{\dist(x, \text{Query}(x))^2}{c^2\cdot\mtdist(x, S)^2}\}$ \algorithmicdo
  	    \STATE ${}$\hspace{1em} $S \leftarrow S \cup x$ 
  	    \STATE ${}$\hspace{1em} \mtopen$(x)$ 
        \STATE ${}$\hspace{1em} Insert(x)  
    \ENDWHILE
    \ENSURE $S$
  \end{algorithmic}
\end{algorithm} 
\end{minipage}
\end{wrapfigure}
In the \rejsam algorithm (\cref{alg:rejsam}), the probability on Line $5$ is not defined for the case that $S$ is an empty set, i.e., the first iteration of the loop. In this case we assume that this probability is one and the sampled element will be added to $S$. We start be presenting a few properties of \rejsam algorithm.
%
%
We show that the expected number of the times that the loop (Line $3$) repeats is ${O}(c^2d^2 k)$.
To that end, we first show that the probability of opening a center in $x$ in any iteration is independent of the \mt embedding and only depends on the LSH data structure. 
This holds, intuitively, because when we sample a point $x$ by calling $\mtsample()$ we then decide to add it based on the distance to the point reported by the LSH data structure which removes the dependency on \mtinit. Specifically, each point $x$ is first sampled w.p. $\frac{\mtdist(x, S)^2}{\sum_{y\in P} \mtdist(y, S)^2}$ and then added to set $S$ w.p. $\frac{\dist(x, \text{Query}(x))^2}{c^2\cdot\mtdist(x, S)^2}$. Therefore, the probability of adding $x$ to $S$ is proportional to $\dist(x, \text{Query}(x))^2$ and we get (see
\if\fullversion1
\cref{app:rejsamperf}
\else
the full version
\fi
for a formal proof):
\begin{restatable}{lemma}{rejsamlemma} \label{lemma:dist}
The probability of inserting a point $x$ to set $S$ in \rejsam algorithm is independent of \mtinit and is equal to $1/n$ for the first iteration and  $\frac{\dist(x, \text{Query}(x))^2}{\sum_{y\in P} \dist(y, \text{Query}(y))^2}$ for other iterations.
\end{restatable}
The main ingredient in the running time analysis is to bound the number of repetitions of the loop (Line $3$). This is roughly done by arguing that  the probability that we add an element $x$ to $S$ after its sampled using \mtsample is $\Omega(\frac{1}{c^2d^2})$ in expectation. Indeed, from \cref{lemma:approxdist} we expect that $\mtdist(x, S)^2 \leq O(d^2\dist(x, \text{Query}(x))^2)$, so $\frac{1}{\Omega(c^2d^2)} \leq \frac{\dist(x, \text{Query}(x))^2}{c^2\cdot\mtdist(x, S)^2}$. Therefore the probability of passing Line 5 is at least $\frac{1}{\Omega(c^2d^2)}$. It follows that, in expectation,  ${O}(c^2d^2k)$ repetitions suffices to add $k$ points to $S$. The formal proof is presented in
\if\fullversion1
\cref{app:rejsam-approx}.
\else
the full version.
\fi
\begin{restatable}{lemma}{rejsamlooprep} \label{lemma:numberiterations}
The expected number of the times that the loop (Line $3$) is repeated is of ${O}(c^2d^2k)$.
\end{restatable}

Putting the discussed ingredients and the approximation ratio analysis together, we get the following result, the proof is presented in
\if\fullversion1
\cref{app:rejsam-approx}.
\else
the full version.
\fi
\begin{restatable}{theorem}{rejsammain}
For any constant $c>1$, with probability at least $(1-1/n)$ \rejsam always samples points $x$ that are at most a factor $c^2$ away from the $D^2$-distribution, its expected running time is $O\left(n \log(d\Delta)(d +  \log n) + k c^2 d^3\log(\Delta) \cdot\left(n\log(\Delta)\right)^{O(1/c^2)}\right)$, and it returns a solution that in expectation is a $O(c^6 \log k)$-approximation of the optimal solution.
\label{thm:rejsam_main}
\end{restatable}

We remark that the runtime can be improved  in the case of a large  $d$ by first applying a dimensionality reduction~\cite{BBCGS19,MMR19} that reduces the dimension of the input points to $O(\log n)$ in time $O(n d \log n)$ and maintains the cost of any clustering up to a constant factor. These works actually prove that the dimension can be reduced to $O(\log k)$. However, by using $O(\log n)$ our algorithm can output the solution for \emph{all} $k=1, 2,\ldots, n$ in near-linear running time $\widetilde{O}\left(nd + n \log \Delta +c^2k\log(\Delta) (n \log \Delta)^{O(1/c^2)}\right)$ (where $\widetilde{O}$ suppresses logarithmic terms in $n$) while maintaining the same asymptotic approximation guarantee as \kmpp. Selecting $\varepsilon = O(1/c^2)$ then yields the following
\begin{corollary}
\label{cor:main}
    For $\varepsilon >0$, there is an $O_\varepsilon(\log k)$-approximation algorithm for the $k$-means problem with a running time of $\widetilde\Theta(nd + (n\log(\Delta))^{1+\varepsilon})$.
\end{corollary}

\section{Empirical Evaluation}
\begin{table*}[t]
  \centering
  \resizebox{\columnwidth}{!}{%
  \begin{tabular}{|c|c|c|c|c|c|c|}
    \hline
    Algorithm & $k=100$ & $k=500$ & $k=1000$ & $k=2000$ & $k=3000$ & $k=5000$ \\
    \hline
    \fastkm & 1.0x  & 1.0x  & 1.0x & 1.0x & 1.0x & 1.0x \\
    \hline  
    \rejsam & 1.04x & 1.09x & 1.04x & 1.07x &  1.01x & 1.28x  \\
    \hline  
    \kmpp  & 0.66x & 3.11x & 6.58x & 15.26x & 18.58x & 42.64x\\
    \hline  
    \kmc   & 0.89x & 1.88x & 3.80x & 8.5x & 16.61x & 38.7x\\
    \hline  
  \end{tabular}
  }
  \smallskip
  
  \caption{Running time of the algorithms divided by the running time of \fastkm for the \texttt{KDD-Cup} dataset. 
  This shows the speed-up that we achieve compared to the \kmpp and \kmc.}
  \label{newtable:kdd_running}
\end{table*}

\begin{table*}[t]
  \centering
  \resizebox{\columnwidth}{!}{%
  \begin{tabular}{|c|c|c|c|c|c|c|}
    \hline
    Algorithm & $k=100$ & $k=500$ & $k=1000$ & $k=2000$ & $k=3000$ & $k=5000$ \\
    \hline
    \fastkm & 1.0x  & 1.0x  & 1.0x & 1.0x & 1.0x & 1.0x \\
    \hline  
    \rejsam & 0.99x & 1.03x & 0.98x & 1.03x & 1.04x & 1.04x \\
    \hline  
    \kmpp   & 0.76x & 4.55x & 8.89x & 16.98x & 23.03x & 46.26x\\
    \hline  
    \kmc    & 0.62x & 1.02x & 1.35x & 2.81x & 4.98x & 8.71x\\
    \hline 
  \end{tabular}
  }
  \smallskip
  
  \caption{Running time of the algorithms divided by the running time of \fastkm for the \texttt{Song} dataset.
  This shows the speed-up that we achieve compared to the \kmpp and \kmc.}
  \label{newtable:song_running}
\end{table*}

\begin{table*}[t]
  \centering
  \resizebox{\columnwidth}{!}{%
  \begin{tabular}{|c|c|c|c|c|c|c|}
    \hline
    Algorithm & $k=100$ & $k=500$ & $k=1000$ & $k=2000$ & $k=3000$ & $k=5000$ \\
    \hline
    \fastkm & 1.0x  & 1.0x  & 1.0x & 1.0x & 1.0x & 1.0x \\
    \hline  
    \rejsam & 0.98x & 1.26x & 1.17x & 1.07x &  1.0x & 0.95\\
    \hline  
    \kmpp   & 0.89x & 4.78x & 8.92x & 14.18x & 23.57x & 36.69 \\
    \hline  
    \kmc    & 0.76x & 0.77x & 1.12x & 1.15x & 1.54x & 2.56x\\
    \hline  
  \end{tabular}
  }
  \smallskip
  
  \caption{Running time of the algorithms divided by the running time of \fastkm for the \texttt{Census} dataset. 
  This shows the speed-up that we achieve compared to the \kmpp and \kmc.}
  \label{newtable:census_running}
\end{table*}

\begin{table*}[t]
  \centering
  \resizebox{\columnwidth}{!}{%
  \begin{tabular}{|c|c|c|c|c|c|c|}
    \hline
    Algorithm & $k=100$ & $k=500$ & $k=1000$ & $k=2000$ & $k=3000$ & $k=5000$ \\
    \hline
    \fastkm& 30335 & 5771 & 2957 & 1582 & 1070 & 640 \\
    \hline  
    \rejsam& 29243 & 5857 & 2999 & 1581 & 1095 & 642 \\
    \hline  
    \kmpp & 24552 & 5128 & 2695 & 1423 & 968 & 562\\
    \hline  
    \kmc  & 25598 & 5384 & 2883 & 1512 & 1045 & 622\\
    \hline  
    \random & 148594 & 51692 & 26199 & 15927 & 13922 & 10017\\
    \hline  
  \end{tabular}
  }
  \smallskip
  
  \caption{Costs of the solutions produced by the algorithm for 
    \texttt{KDD-Cup} dataset for various values of $k$. All the numbers are scaled down by a factor $10^3$.}
    
  \label{newtable:kdd_cost}
\end{table*}

\begin{table*}[t]
  \centering
  \resizebox{\columnwidth}{!}{%
  \begin{tabular}{|c|c|c|c|c|c|c|}
    \hline
    Algorithm & $k=100$ & $k=500$ & $k=1000$ & $k=2000$ & $k=3000$ & $k=5000$ \\
    \hline
    \fastkm & 21898668 & 16732379 & 14987614 & 13477854 & 12691185 & 11628744 \\
    \hline  
    \rejsam & 21743137 & 16851767 & 15024812 & 13558210 & 12720314 & 11654493\\
    \hline  
    \kmpp   & 21583261 & 16409834 & 14746899 & 13395052 & 12480900 & 11496421\\
    \hline  
    \kmc    & 21596184 & 16344430 & 14750601 & 13246450 & 12450688 & 11476712\\
    \hline 
    \random & 23255642 & 17919981 & 16373134 & 14579718 & 13934375 & 12938255\\
    \hline  
  \end{tabular}
  }
  \smallskip
  
  \caption{Costs of the solutions produced by the algorithms for 
    the \texttt{Song} dataset for various values of $k$. All the numbers are scaled down by a factor $10^5$.}
    \label{newtable:song_cost}
\end{table*}

\begin{table*}[t]
  \centering
  \resizebox{\columnwidth}{!}{%
  \begin{tabular}{|c|c|c|c|c|c|c|}
    \hline
    Algorithm & $k=100$ & $k=500$ & $k=1000$ & $k=2000$ & $k=3000$ & $k=5000$ \\
    \hline
    \fastkm& 17304 & 9820 & 7883 & 6326 & 5625 & 4868  \\
    \hline  
    \rejsam& 17735 & 9970 & 8031 & 6432 & 5644 & 4893 \\
    \hline  
    \kmpp  & 18498 & 9585 & 7812 & 6254 & 5561 & 4815 \\
    \hline  
    \kmc   & 17242 & 9844 & 7710 & 6272 & 5595 & 4838\\
    \hline  
    \random& 19912 & 10630 & 8678 & 6880 & 6120 & 5228 \\
    \hline  
  \end{tabular}
  }
  \smallskip
  
  \caption{Costs of the solutions produced by the algorithm for 
    \texttt{Census} dataset for various values of $k$. All the numbers are scaled down by a factor $10^4$.}
    
  \label{newtable:census_cost}
\end{table*}

In this section we empirically validate out theoretical results by comparing our algorithms \fastkm and \rejsam (see the details on how we set the parameters for LSH in 
\if\fullversion1
\cref{appx:LSHparams}
\else
the full version
\fi
)  with the following two baselines:\\
\kmpp algorithm: Perhaps the most commonly used algorithm in this field. It samples $k$ points according to the $D^2$-distribution.\\
\kmc algorithm: A recent result \cite{NIPS2016_6478} based on random walks that improves the running time of the \kmpp algorithm while maintaining a (weaker) theoretical guarantee on the solution quality.
\paragraph{Datasets, Experiments, and Setup}
We ran our algorithms on three classic datasets from UCI library~\cite{Dua:2019}:
\texttt{KDD-Cup}~\cite{KDD} ($311,029$ points of dimension $74$)
and \texttt{song}~\cite{bertin2011million} ($515,345$ points of dimension $90$) \texttt{Census}~\cite{DBLP:conf/kdd/Kohavi96} ($2,458,285$ points of dimension $68$).
We did not apply any dimensionality
reduction technique for any of the algorithms; all the considered data sets is of small dimension.
We compare the quality of the clustering, i.e., the cost of the objective function, along with their running times.
For the \kmc algorithm, we used the code provided by the authors with the same parameter suggested there, i.e., $m=200$.\footnote{$m$ is the number of steps in the random walk.} 
The algorithms were run on a standard desktop computer. 

\paragraph{Discussion}
Our results show that the algorithms we propose are much faster than
both the baselines, i.e., \kmpp and \kmc, as $k$ grows. For large $k = 5000$, it is an order of magnitude faster than both \kmpp and \kmc. Moreover,  
 the running time of our algorithms is already significantly faster than both baselines for moderate values of $k$ such as $k=500$ for \texttt{KDD-Cup} and $k=1000$ for \texttt{Song} and \texttt{Census}. We refer to  Tables~\ref{newtable:kdd_running}, \ref{newtable:song_running}, and \ref{newtable:census_running} for more details.

Importantly, we achieve this improvement in the running time without making any significant sacrifice to solution quality from both a theoretical and experimental perspective.  While the solution quality is sometimes worse by $10$-$15$\% for small $k$, the $k$-means costs of the solutions produced by \fastkm and \rejsam algorithms are comparable (overall almost the same)  with the baselines for all the experiments for moderate values of $k\geq 1000$. This is in contrast to the simplest seeding algorithm \random which selects the $k$ centers uniformly at random from the input data set. While \random clearly provides for a very fast seeding algorithm, it does so by significantly deteriorating the solution quality. This can e.g. be seen in our results for the \texttt{KDD-Cup} dataset  where \random consistently gives solutions of much worse quality. For more details, see Tables~\ref{newtable:kdd_cost},~\ref{newtable:song_cost}, and~\ref{newtable:census_cost} where the solution costs are given. 
The variance along with experimental setting is reported in 
\if\fullversion1
\cref{app:var}.
\else
the full version.
\fi

\section{Conclusions}
In this paper we present new efficient algorithms for $k$-means++ seeding. Our algorithms outperform previous work as $k$ grows  and come with strong theoretical guarantees. Interesting avenues for future work are to develop efficient distributed algorithms for the same problem and to prove lower bounds on the running time.

\section*{Broader Impact} 
Our work focuses on speeding-up the very popular \kmpp algorithm
for clustering. The \kmpp algorithm is used in a variety of domains and is 
an important tool for extracting information, compressing data, or unsupervised classification tasks. Our result shows
that one can obtain a much faster implementation of the $k$-means++ algorithm
while preserving its approximation guarantees both in theory and in practice. 
Therefore, we expect that our new algorithm could have impact in several domains
in which clustering plays an important role. 
A broader concrete impact in society is harder to predict since this is mainly fundamental research.

\begin{ack}
The last author is supported by the Swiss National Science Foundation project 200021-184656 ``Randomness in Problem Instances and Randomized Algorithms.”
\end{ack}

\bibliography{references}
\bibliographystyle{plain}

\if\fullversion1

\appendix
\section{Proofs Omitted from~\cref{sec:multitree}} \label{app:multitreeperf}
\lemmamultitree*

\begin{proof}
We start by showing that $\dist(p, S)^2 \le \mtdist(p,S)^2$. We prove this for every single tree, which implies the result. In particular, we show that for any tree and for any two points $p,q \in P$, we have $\dist(p, q) \le \tdist(p,q)$. Assume that the lowest common ancestor of $p, q$ is at height $ i$. Therefore in each dimension, they differ at most by $\frac{2\cdot \MaxDist}{2^i}$ since the side length of the cube at this height is $\frac{2\cdot \MaxDist}{2^i}$, therefore
\begin{align*}
\dist(p,q) \leq \sqrt{d}\cdot \frac{2\cdot \MaxDist}{2^i}.
\end{align*}
Moreover, \tdist$(p,q)$ is defined as the length of the shortest path between them and the length of the edge at height $j$ is $\sqrt{d}\cdot \frac{\MaxDist}{2^j}$ for $0 \leq j < H$. So 
\begin{align*}
\tdist(p,q) = 2 \sum_{i \leq j < H} \sqrt{d} \cdot \frac{\MaxDist}{2^j} = 2\sqrt{d}\cdot \MaxDist \sum_{i \leq j < H} 2^{-j}\geq \sqrt{d}\cdot \frac{2\cdot \MaxDist}{2^i}.
\end{align*}
Therefore,
$
\dist(p,q) \leq \sqrt{d}\cdot \frac{2\cdot \MaxDist}{2^i} \leq \tdist(p,q),
$
which concludes the proof of the first part of the lemma.

Now we focus on the second part of the lemma, i.e., $\EE[\mtdist(p,q)^2] \leq O(d^2 \cdot \dist(p, q)^2)$. Let $p = (p_1, \ldots ,p_d)$ and $q = (q_1,\ldots, q_d)$ be two points in $\mathbb{R}^d$. We first analyze the probability that these two points are separated at a certain height $i$ in a single tree. In a single tree, two points are separated at height $i$ if they are separated in at least one of the coordinates. The probability that $p$ and $q$ are separated in the $j$-th dimension is at most $\frac{|p_j-q_j|}{\frac{2\cdot \MaxDist}{2^i}}$. Therefore, if we let $s_i$ denote the probability that they are separated at height $i$  (but at no smaller height), then by the union bound 
$
s_i \leq \sum_{1 \leq j \leq d} \frac{|p_j-q_j|}{\frac{2\cdot\MaxDist}{2^i}} \leq\sqrt{d} \cdot \frac{2^i\dist(p,q)}{2\cdot \MaxDist}\, ,
$ 
where the second inequality  holds because $\sum_{1 \leq j \leq d}  |p_j-q_j|\leq\sqrt{d}\cdot \dist(p,q)$. So the probability that they are separated at height $i$ or before is at most
\[
\sum_{0\leq j \leq i} s_j \leq \sum_{0\leq j \leq i} \sqrt{d}\cdot \frac{2^j\dist(p,q)}{2\cdot \MaxDist} \leq \sqrt{d}\cdot\frac{2^{i+1}\dist(p,q)}{2\cdot\MaxDist}\, .
\]

Notice that, as before, we also have that $\tdist(p,q) = 2\sqrt{d}\cdot \MaxDist \cdot \sum_{i \leq j < H} 2^{-j} \leq \sqrt{d}\cdot\frac{4\cdot \MaxDist}{2^i}$ in the case that $p,q$ are separated at height $i$. 
Now recall that \mtdist is the minimum distance among all the three tree embeddings, so it is enough that the two points are separated at height $i$ in a single tree to be at this distance (in the other two trees they can be separated in a level closer to the root). There are three ways to select the tree of minimum distance and so
\begin{align*}
 \EE[\mtdist(p,q)^2] &\leq  
3\sum_{ 0\leq i < H} s_i \left(\sum_{0\leq j \leq i} s_j\right)^2 \cdot \left(\sqrt{d}\cdot\frac{4\cdot \MaxDist}{2^i} \right)^2 \\
& \leq 3 \sum_{ 0\leq i < H} s_i \left(\sqrt{d}\cdot\frac{2^{i+1}\dist(p,q)}{2\cdot\MaxDist}\right)^2 \cdot \left(\sqrt{d}\cdot\frac{4\cdot \MaxDist}{2^i} \right)^2 \\
&= 3 \sum_{ 0\leq i < H} s_i 16d^2 \dist(p,q)^2  
 =  48d^2 \dist(p,q)^2  \sum_{ 0\leq i < H} s_i  \\
&\leq  48 d^2\dist(p,q)^2 = O(d^2 \dist(p,q)^2)\,,
\end{align*}
where the last inequality holds because $\sum s_i \leq 1$ since the $s_i$'s denote the probabilities of mutually disjoint events. 
\end{proof}

\section{Proof Omitted from \cref{sec:fastkmeans}} \label{app:fastkmeansperf}
 \treeopenrunningtime*
 \begin{proof}
As the multi-tree embedding consists of three trees, it is sufficient to analyze the running time of the for-loop at Step~1 for  a single tree $T$. 
In a single call to \mtopen we have that Steps~2-~4 runs in time $O(\log(d\Delta))$ since each tree in the multi-tree embedding has depth at most $H= \log(d\Delta)$. Hence the total running time for these steps across the $k$ calls to $\mtopen$ is $O(k \log(d\Delta))$.

The for-loop at Step~5 can be implemented as follows. Observe that the distance $\tdist_T(x,y)$ for a point $y \in P_T(v_0)$ equals $0$. Furthermore, for $i = 1, 2, \ldots, \ell$  and a point $y\in P_T(v_i) \setminus P_T(v_{i-1})$, the distance $\dist_T(x,y)$ equals twice the length of the path in $T$ from $v_0$ to $v_i$. We can thus calculate all relevant distances in time $O(\log(d\Delta))$ (and in time $O(k \log(d\Delta))$ across all $k$ calls). 

The weights of the points can then be updated by first considering the points in $P_T(v_0) $, then those in $P_T(v_1) \setminus P_T(v_0)$, and so on until the points in $P_T(v_\ell) \setminus P_T(v_{\ell-1})$. The running time of Step~7 is thus $O(\sum_{i=0}^\ell |P_T(v_i)|)$. Moreover, as each execution of Step~8 takes time $O(\log n)$ (since the sample-tree is balanced binary tree with $n$ leaves and thus of height $O(\log n)$), the running time of this step is $O(\sum_{i=0}^\ell |P_T(v_i)| \cdot \log n)$. Now a key point is that a node in a tree in the multi-tree embedding can only be marked once. Therefore, using that $\sum_v |P_T(v)| = O(n \log(d\Delta))$, we have that the total running time of the for-loop is $O(n \log(d\Delta) \log n)$.  
The total running time of the $k$ calls to $\mtopen$ is therefore  $O(k \log n + n \log(d\Delta) \log n)$.
\end{proof}

\section{Proofs Omitted from \cref{sec:rej-samp}} \label{app:rejsamperf}
\rejsamlemma*

\begin{proof}
The proof is by induction on the size of the set $S$. The base case is clear since, as aforementioned, in Line $5$ we accept the point with probability one and \mtsample() returns each point with probability $1/n$. 
Now assume that lemma holds until $\ell$ centers are added (i.e., $|S| = \ell$). It follows from the induction hypothesis the current set $S$ is independent from the \mt initialization since all the elements added so far are independent.
 By the assumption that the data structure is successful,  the above minimum on Line $5$ is always attained by the second term, i.e., we have $\frac{\dist(x, \text{Query}(x))^2}{c^2\cdot\mtdist(x, S)^2} \leq 1$.
   Indeed,  $\dist(x, \text{Query(x)})^2 \leq c^2\cdot  \dist(x, S_i)^2 \leq c^2\cdot \mtdist(x,S_i)^2$, where  the first inequality is by the success of the data structure and the second inequality is by the fact that the multi-tree embedding only increases distances (see~\cref{lemma:approxdist}).

Therefore in \rejsam algorithm, the probability that a point $x$ is sampled is 
\begin{align*}
\frac{\mtdist(x, S)^2}{\sum_{y\in P} \mtdist(y, S)^2}  
 \cdot \frac{\dist(x, \text{Query}(x)^2)}{c^2\cdot\mtdist(x, S)^2} \\
\qquad = \frac{\dist(x, \text{Query}(x))^2}{c^2\cdot \sum_{y\in P} \mtdist(y, S)^2}\, . 
\end{align*}
Notice that the denominator does not depend on the point $x$ and one can think of it as a constant term. We repeat the sampling process until we pick a point. Therefore, the probability of choosing any point $x$ is
$
\frac{\dist(x, \text{Query}(x))^2}{\sum_{y\in P} \dist(y, \text{Query}(y))^2}\,,
$
independent of the the $\mt$ initialization. This completes the inductive step and  concludes the proof of the lemma.
\end{proof}

\rejsamlooprep*
\begin{proof}
    Let $R$ be the random variable that equals the number of times that the loop is repeated. We further divide $R$ into $R_0, \ldots, R_{k-1}$ where, for $i\in \{0,\ldots, k-1\}$, $R_i$ denotes the number of times the loop is executed when the set $S$ of opened centers has size $i$, i.e., when exactly $i$ centers have been opened. Then $R= R_0 + R_1+ \ldots + R_{k-1}$ and, by linearity of expectation,
    \begin{align*}
        \ee{R} &= 
         \ee{R_0} + \ee{R_1} + \ldots + \ee{R_{k-1}}\,.
    \end{align*}
    We have $\ee{R_0} = 1$ since the first center is selected uniformly at random and it is always opened, i.e., added to $S$. We complete the proof by proving
    \begin{align*}
        \ee{R_i} \leq O(c^2 \cdot d^2)\qquad \mbox{for $i\in \{1, \ldots, k-1\}$.}
    \end{align*}
    Consider $R_i$ and let $S_i$ be the set containing the first $i$ centers that were opened by the algorithm.
    We actually prove the stronger statement that $\ee{R_i} \leq O(c^2 \cdot d^2)$ no matter the set $S_i$. 
    
   Consider an  iteration of the loop. First, 
  as argued in the proof of~\cref{lemma:dist}, 
  %
   the probability that an iteration of the loop results in adding a point $x$ to the set of opened centers equals
       $ \frac{1}{c^2} \cdot \frac{\sum_{x\in P} \dist(x, \text{Query}(x))^2}{\sum_{y\in P} \mtdist(y, S_i)^2}\,$.
   If we let $q$ denote this probability  then
   $
       \ee{R_i} = \sum_{t=1}^\infty t q \cdot (1-q)^{t-1}\,,
   $
   which equals $1/q$. 
   
   We thus have
   \begin{align*}
       \ee{R_i} = 1/q = c^2 \cdot 
       \frac{\sum_{y\in P} \mtdist(y, S_i)^2}{\sum_{x\in P} \dist(x, \text{Query}(x))^2}\,,
   \end{align*}
   for a fixed multi-tree embedding.
   
   The lemma now follows from that $\dist(x, \text{Query}(x))^2 \geq \dist(x, S_i)^2$ and from~\cref{lemma:dist} which says that the distribution of the random multi-tree embedding is independent from $S_i$. We thus have, by also taking the expectation over  the random multi-tree embedding (see~\cref{lemma:approxdist}),  that 
   \begin{align*}
       \ee{R_i} \leq c^2 \cdot \frac{O(d^2) \sum_{y\in P} \dist(y, S_i)^2}{\sum_{x\in P} \dist(x, \text{Query}(x))^2} = O(c^2 \cdot d^2)\,.
   \end{align*}

\end{proof}

\section{LSH data structure}
\label{sec:LSH-data-structure}

In this section we describe the data structure guaranteed by~\cref{thm:data_structure}. It follows the construction first introduced in~\cite{IndykM98}. Their construction is based on locality-sensitive hash families:
\begin{definition}[Locality-sensitive hashing] Let $\cH$ be a family of hash functions mapping $\mathbb{R}^d$ to some universe $U$. We say that $\cH$ is  $(R, cR, p_1, p_2)$-sensitive if \emph{for any} $p, q\in \mathbb{R}^d$ it satisfies the following properties:
\begin{itemize}
    \item If $\|p-q\|_2 \leq R$ then $\Pr_\cH[h(p) = h(q)] \geq p_1$.
    \item If $\|p-q\|_2 \geq cR$ then $\Pr_\cH[h(p) = h(q)] \leq p_2$.
\end{itemize}
\end{definition}

The specific family of hash functions that we use is by~\cite{andoni2006near}. We summarize the main properties of their family in the following theorem. Here, and in the following, we denote by $n$ the size of the data set $P \subseteq \mathbb{R}^d$. 
\begin{theorem}[\cite{andoni2006near}]
For any $R>0$ and $c>1$, there exists a family $\cH$ of hash functions for $\mathbb{R}^d$ with the following properties:
\begin{itemize}
    \item $\cH$ is $(R, cR, p_1, p_2)$-sensitive with $\frac{\log(1/p_1)}{\log(1/p_2)} = 1/c^2 + o(1)$ and $\frac{1}{\log(1/p_2)} = O(1)$.
    \item The time to compute $h(p)$ for $h\in \cH$ and $p\in \mathbb{R}^d$ is $O(dn^{o(1)})$.
\end{itemize}
\label{thm:LSH_hash_family}
\end{theorem}

We first describe a data structure for the ``gap version''. Then we show, using standard arguments, that this gives the data structure as stated in~\cref{thm:data_structure}.

\subsection{Monotone data structure for gap version}
In this section we are going to develop a data structure that is parameterized by $c\geq  1 $ (the accuracy) and $R> 0$ (the scale). We refer to it as the $(c,R)$-gap data structure. It is different from the data structure guaranteed by~\cref{thm:data_structure} as it only have guarantees that depend on the scaling parameter $R$ (see the statement of~\cref{thm:cnn_data_structure} below).

\paragraph{Selection of parameters.} 
We let $\cH$ be the $(R, cR, p_1,p_2)$-sensitive hash family for $\mathbb{R}^d$ given by~\cref{thm:LSH_hash_family}. We also let $\delta > 0$ be a parameter of our data structure that determines the probability of failure (and impacts the running time).  Other parameters that we use are now determined as follows: 
\begin{itemize}
    \item $\eta = \left(\frac{\delta}{n}\right)^\frac{3}{1-\rho}$ where $\rho = \frac{\log(1/p_1)}{\log(1/p_2)}$, 
    \item $m = \frac{\log(1/\eta)}{\log (1/p_2)}  = O(\log(1/\eta))$, and
    \item $\ell = 100 \cdot \log(1/\eta) \cdot (1/\eta)^\rho$. 
\end{itemize}

\paragraph{Description of data structure.} The data structure is based on $\ell$ hash tables $T_1, T_2, \ldots, T_\ell$ (with linked lists at each entry to deal with collisions). The  $\ell$ hash functions $f_1, f_2, \ldots, f_\ell$ for these tables are constructed from $\cH$ as follows: for $i\in \{1, 2, \ldots, \ell\}$, $f_i$ is obtained by selecting $m$ \emph{independent samples} $h_{i,1}, h_{i,2}, \ldots, h_{i,m}$ from $\cH$. That is, $f_i$ is a $m$-dimensional hash function defined by
\begin{align*}
    f_i(p) & = [h_{i,1}(p), h_{i,2}(p), \ldots, h_{i,m}(p)] \qquad \mbox{for $p\in \mathbb{R}^d$.}
\end{align*}

We are now ready to define the operations Insert and Query:
\begin{itemize}
    \item Insert($p$): A point $p\in P \subseteq \mathbb{R}^d$ is inserted in each of the $\ell$ hash tables  by appending the point at \emph{the end} of the linked list associated to the entry $T_i[f_i(p)]$ for $i=1,2,\ldots, \ell$.
    \item Query($p$): For each $i\in \{1,2, \ldots, \ell\}$,  let $q_i$ be the first element (if any) in the linked list $T_i[f_i(p)]$ that satisfies $\dist(p, q_i) \leq cR$. This gives up to $\ell$ candidate points, one for each hash table. Among these candidate points, \emph{output the one with the minimum distance to $p$} (or output none if no candidate point is found in any of the hash tables).
\end{itemize}

This completes the description of the $c$-NN data structure and we proceed to its analysis.

\paragraph{Analysis.} We show that the described data structure satisfies the following guarantees:
\begin{theorem}
    For a data set $P\subseteq \mathbb{R}^d$ of $n$ points, the data structure with error parameter $\delta > 0$ satisfies the following guarantees: 
    \begin{enumerate}
        \item The Insert operation runs in time $O\left(d\cdot (n/\delta)^{O(1/c^2)}\right)$. 
        \item With probability at least $1-\delta$, the Query operation satisfies the following. Given $p\in P$, if there exists an inserted point within distance $R$ from $p$, then Query($p$) returns a point $q$ with $\dist(p,q) \leq cR$. Moreover, the running time is 
        time is $(n/\delta)^{O(1/c^2)}$.
    \end{enumerate}
     Furthermore, the data structure  is \emph{monotone under insertions}: the distance between $p$ and  Query($p$) is non-increasing after inserting more points.
     \label{thm:cnn_data_structure}
\end{theorem}
Throughout the analysis we assume that $c$ is a large enough constant and that $n$ is sufficiently large. This is motivated by the fact that otherwise a trivial data structure can achieve the bounds claimed by the theorem. 

The analyses of the monotonicity property and the running time of the insertion operation are rather immediate:
\begin{itemize}
    \item The monotonicity property is by definition of the operations Insert and Query. To see that,  suppose we run Query$(p)$ for a point $p$. We will argue that inserting any new point $p'$ may not increase the distance $\dist(p, \text{Query}(p))$. Indeed, when $p'$ is inserted it is appended to the end of the linked-lists $T_i[f_i(p)]$ for $i = 1,2, \ldots, \ell$. Now when we execute Query($p$) the only way that $p'$ will be one of the candidate points $q_1, q_2, \ldots, q_\ell$ is if it, for some $i\in \{1, 2, \ldots, \ell\}$, is the first point in $T_i[f_i(p)]$ within distance $cR$ from $p$. It follows (since insertions are appended at the end of the linked-lists whereas queries inspects the lists from the beginning) that all the the points that were candidates before the insertion of $p'$ are still candidates. Therefore the distance from $p$ to the minimum distance point (of the candidates) can only decrease after inserting a new point $p'$. 
    \item  We proceed to analyze the running time of the Insert operation. On the insertion of a point $p\in P$, it is appended to each of the $\ell$ linked lists $T_1[f_1(p)], T_2[f_2(p)], \ldots, T_\ell[f_\ell(p)]$. Appending an element to a linked list takes $O(1)$ time whereas the cost of calculating a single hash $f_i(p)$ is $m$ times the cost of calculating $h(p)$ for a single $h\in \cH$, which in turn by~\cref{thm:LSH_hash_family} is $O(dn^{o(1)})$.  
The running time of an insertion is therefore dominated by the time it takes to calculate the $\ell$ hashes $f_1(p), f_2(p), \ldots, f_\ell(p)$, which by the above arguments takes time
\begin{align*}
    \ell \cdot m \cdot O(dn^{o(1)}) & = O\left( \log(1/\eta) \cdot (1/\eta)^\rho\right) \cdot O\left( \log(1/\eta)\right ) \cdot O(dn^{o(1)})\\
    & = O\left( \log(n/\delta)^2 \cdot (n/\delta)^{3\rho/(1-\rho)}\right)\cdot O(dn^{o(1)}) \\
    & =  O\left(d \cdot (n/\delta)^{O(1/c^2)}\right),\\
\end{align*}
where we used that $c$ is a large enough constant for the last equality. 
\end{itemize}

We proceed to analyze the Query operation which requires a little more work. In order to guarantee that Query returns a nearby point if one exists, we need bound the probability of having a false negative. On the other hand, to bound the running time of the Query operation we need to bound the false positives. 
The following two lemmas bounds these quantities, starting with the probability of false positives.
\begin{lemma}
    For any $i\in \{1, 2, \ldots, \ell\}$ and two points $p, q\in \mathbb{R}^d$ with $\dist(p,q) \geq cR$, we have
    \begin{align*}
        \Pr[ f_i(p) = f_i(q)] \leq \eta\,.
    \end{align*}
    \label{lemma:false_positive}
\end{lemma}
\begin{proof}
    By the independence of $h_{i,1}, h_{i,2}, \ldots, h_{i,m}$, we have
    \begin{align*}
        \Pr[f_i(p) = f_i(q)] & = \Pr_{h\sim \cH}[h(p) = h(q)]^m \leq p_2^m\,,   
    \end{align*}
    which by the selection of $m$ equals $\eta$.
\end{proof}

\begin{lemma}
    For any two points $p, q\in \mathbb{R}^d$ with $\dist(p,q) \leq R$, we have
    \begin{align*}
          \Pr[\exists i \mid f_i(p) = f_i(q)]  \geq 1-\eta\,.  
    \end{align*}
    \label[lemma]{lemma:false_negative}
\end{lemma}
\begin{proof}
 Similar to the calculations in the proof of the previous lemma, we have
\begin{align*}
    \Pr[\exists i \mid f_i(p) = f_i(q)] &= 1 - \Pr[\forall i, f_i(p) \neq f_i(q)] \\
    &= 1 - \Pr[f_i(p) \neq f_i(q)]^{\ell} \\
    &\geq 1 - (1 - p_1^m)^{\ell}\,.
\end{align*} 
    By the definition of $\rho$, $p_1 = p_2^\rho$ and so 
\begin{align*}
    \Pr[\exists i \mid f_i(p) = f_i(q)] 
    &\geq 1 - (1 - p_2^{\rho m})^{\ell} \\ 
    & =  1 - \left(1- \eta^{\rho}\right)^\ell  \\
    & \geq 1- \eta\,,
\end{align*} 
where the last inequality is by the selection of $\ell$.
\end{proof}


Equipped with these two lemmas we are now ready to analyze the Query operaton. Specifically, we have the following corollary:
\begin{corollary}
    Consider a set $P$ of $n$ points in $\mathbb{R}^d$. Then with probability at least $1- \delta$ we have that the hash functions $f_1, f_2, \ldots, f_\ell$ satisfy the following:
    \begin{itemize}
        \item For any $p,q\in P$ with $\dist(p,q) \geq cR$, we have $f_i(p) \neq f_i(q)$ for all $i\in \{1,2, \ldots, \ell\}$.
        \item For any $p,q \in P$ with $\dist(p,q) \leq R$, we have that there is an $i\in \{1,2,\ldots, \ell\}$ such that $f_i(p) = f_i(q)$. 
    \end{itemize}
\end{corollary}
Before giving the proof of the corollary, note that the first property implies that we have no false positives. Therefore, the running time of Query is the same as for Insertion: it is dominated by the time to calculate the $\ell$ hash functions which is $O\left(d\cdot (n/\delta)^{O(1/c^2)}\right)$. Moreover, the second property guarantees that we always have a hash collision when there is a nearby point of the query-point $p$. This implies that Query($p$) returns a point $q$ with $\dist(p,q) \leq cR$ if there is a point within distance $R$ from $p$ that has been inserted.  To complete the proof of~\cref{thm:cnn_data_structure} it thus remains to prove the corollary:
\begin{proofof}{Corollary}
    We show that each of the two properties hold fail probability at most $\delta/2$. The statement then follows by the union bound. 
    
    For the first property, there are $\ell$ hash functions and at most ${n\choose 2} \leq n^2$ pairs $p,q\in P$ such that $\dist(p,q) \geq cR$. Therefore, by~\cref{lemma:false_positive} and the union bound, we have that the first property fails with probability at most
    \begin{align*}
        \ell \cdot n^2 \cdot \eta & =  \left( 100 \cdot \log(1/\eta) \cdot (1/\eta)^\rho \right) \cdot n^2 \cdot  \eta \\
        & =  \left( 100 \cdot \log(1/\eta) \right) \cdot n^2 \cdot  \eta^{1-\rho} \\
        &= \left( \frac{300}{1-\rho} \cdot \log(n/\delta) \right) \cdot n^2 \cdot \left(\delta/n)\right)^3 \\
        & \leq \delta/2\,,
    \end{align*}
    where for the last inequality we used that $n$ and $c$ are large.
    
    For the second property, there are at most ${n\choose 2} \leq n^2$ pairs $p,q\in P$ such that $\dist(p,q) \leq R$. So by the union bound and~\cref{lemma:false_negative}, we have that the second property fails with probability at most
        $n^2\cdot \eta$ which by the above calculations is at most $\delta/2$.
\end{proofof}

\subsection{Putting everything together: Proof of~\cref{thm:data_structure}}

The proof of~\cref{thm:data_structure} now follows from~\cref{thm:cnn_data_structure} by standard arguments. Again we assume that $c$ is a large constant (since otherwise a trivial data structure will satisfy the properties of the theorem).  

Recall that all distances are between $\MaxDist/(2\Delta)$ and $\MaxDist$. We make $\log(2\Delta)$ many copies of the gap data structure guaranteed by~\cref{thm:cnn_data_structure}. Each of the copies will have an error parameter $\delta = \frac{1}{n \log (2\Delta)}$ and  the $i$:th copy will have parameters $(c_i, R_i)$ with $c_i = c/2$ and $R_i = 2^{i-1} \MaxDist/(2\Delta)$. The operations now work as follows:
\begin{itemize}
    \item Insert($p$): the point $p\in P$ is inserted into each  of the $\log(2\Delta)$ copies of the gap data structure;
    \item Query($p$): we query the point $p$ in each of the $\log(2\Delta)$ copies and out of the returned points, we return the closest to $p$.
\end{itemize}
Since the gap data structure of~\cref{thm:cnn_data_structure} is monotone we have that the resulting data structure satisfies  monotonicity. That it succeeds with probability at least $1-1/n$ follows from the selection of $\delta$ and the union bound over $\log(2\Delta )$ many copies of the gap data structure. Furthermore the guarantees of the query operation (to find a nearby point) is satisfied: let $q$ be the closest point to $p$ and suppose that $2^{i-2} \cdot \MaxDist/(2\Delta)\leq \dist(p,q) \leq 2^{i-1} \cdot \MaxDist/(2\Delta)$ . Then, on Query($p$), the $i$:th copy of the gap data structure is guaranteed to return a point within distance $c/2\cdot 2^{i-1} \cdot \MaxDist/(2\Delta) = c \cdot 2^{i-2} \cdot \MaxDist/(2\Delta) \leq c \cdot \dist(p,q)$ of $p$.  Finally the running time of the operations is $\log(2\Delta)$ times the running time of each operation in the gap data structure. Hence, since $\delta = 1/(n \log(2\Delta))$, the running time of the operations is  ${O}\left(\log(2\Delta) \cdot d\cdot \left(n^2\log(2\Delta )\right)^{O(1/c^2)}\right) = {O}\left( d\log(\Delta)\cdot \left(n\log(\Delta )\right)^{O(1/c^2)}\right)$ as required.

\subsection{LSH Parameters in our Experiments}
\label{appx:LSHparams}
We use the locality sensitive hash families based on 
$p$-stable distribution introduced by Datar et al.~\cite{datar2004locality}.
We set the parameters as follows. We work with one scale, and we set the number
of hash functions to be 15. Moreover, we set the collision parameter (referred
to as $r$ in~\cite{datar2004locality} to be 10).

\section{\rejsam Algorithm Analysis}
In this section we analysis the \rejsam algorithm. We start by proving approximation guarantee and then stating the main theorem.

\subsection{Analysis of Approximation Guarantee} \label{app:rejsam-approx}
In this section we prove that \rejsam has an approximation guarantee of $O(c^6\log(k))$. Hence, for a fixed $c$, it has the same asymptotic approximation guarantee as the standard implementation of \kmpp but with the advantage that  it runs in near-linear time. 
For simplicity we  assume that the LSH data structure is successful throughout the whole analysis. That is, for any $p\in P$, $\text{Query}(p)$ returns a point within distance $c\cdot \delta$  where $\delta$ is the minimum distance from $p$ to a point inserted in the data structure. 

\cref{thm:rejsam_main} says that the probability to sample a center in \rejsam is very close to the same $D^2$-distribution as in \kmpp. At first, it therefore appears rather intuitive that they should have the same approximation guarantee. However, the analysis of \kmpp is rather sensitive to even small perturbations to the probability of sampling a center. Indeed, in a recent paper~\cite{noisykmeanspp}, it was proved that the version of \kmpp where centers are sampled using an approximation of the $D^2$-distribution achieves an approximation guarantee of $O(\log(k)^2)$. 
To get a tight guarantee of $O(\log(k))$ was raised as an open problem. 
Our analysis does not resolve this question. Instead we use the additional \emph{monotonicity} property of our LSH data structure (see~\cref{thm:rejsam_main}) to circumvent the most technical difficulty of~\cite{noisykmeanspp}. This allows us to establish the tight asymptotic approximation guarantee of our procedure. Similarly to the proof in~\cite{noisykmeanspp}, our analysis closely follows  Dasgupta's analysis of \kmpp~\cite{Dasgupta}. The main difference is a slight change of the ''potential'' function (see~\eqref{eq:potential}). However, for the sake of completeness, we reproduce the complete analysis here. 

\paragraph{Notation:} Throughout the proof, we use the following notation. For a set $P'\subseteq P$ of the points and an (ordered) set of centers $S = \{s_1, \ldots, s_i\}$ let
\begin{itemize}
    \item  $\Phi(P', S)$ be the $k$-means cost of data points $P'$ with respect to the centers $S$, i.e.,
    \begin{align*}
        \Phi(P', S) = \sum_{x\in P'} \dist(x, S)^2\,,
    \end{align*}
    \item  $\PhiLSH(P', S)$ be the $k$-means cost of data set $P'$ with respect to  the centers $S$ when using the assignment given by the LSH data structure, i.e.,
    \begin{align*}
        \PhiLSH(P', S) = \sum_{x\in P'} \dist(x, \text{Query}(x))^2\,,
    \end{align*}
    where the points of $S$ have been inserted into the data structure in the order $s_1, s_2, \ldots, s_i$. 
\end{itemize}
Furthermore, we  denote by $\OPT_i(P)$ the cost of an optimal clustering of the data points $P$ using $i$ centers and we let $C_1^*, C_2^*, \ldots, C_k^*$ be the partition of $P$ into $k$ cluster in a fixed optimal solution (with $k$ centers).

\subsubsection{Two preliminary lemmas}
We start our analysis with two preliminary lemmas which are very similar to lemmas in~\cite{noisykmeanspp}, which in turn are based on similar lemmas in  the original \kmpp paper~\cite{arthur2007k}. 

As the first center is chosen uniformly at random in both \rejsam and \kmpp, we can reuse the following statement from the original analysis.
\begin{lemma}[Lemma~$3.1$ in~\cite{arthur2007k}]
    Let $s_1$ denote the first center chosen by \rejsam. For each optimal cluster $C^*_i$,
    \begin{align*}
        \ee{\Phi(C^*_i, \{s_1\}) \mid s_1 \in C^*_i} \leq 2 \cdot \OPT_1(C^*_i)\,.
    \end{align*}
    \label{lem:first_cluster}
\end{lemma}

For the next lemma, we use that \cref{thm:rejsam_main} says that a center $s$ is sampled with a probability in $[q/c^2, q\cdot c^2]$ where $q$ denotes the probability that $s$ would be sampled by the $D^2$-distribution. This allows us to use Lemma~$5$  in the noisy \kmpp analysis:

\begin{lemma}[Lemma~$5$ in~\cite{noisykmeanspp}]
    Consider \rejsam after at least one center has been opened  and let $S \neq \emptyset$ denote the current set of centers. We denote by $s$ the next sampled center. Then for any $S\neq 0$ and any optimal cluster $C_i^*$,
    \begin{align*}
        \ee{\Phi(C_i^*, S \cup \{s\}) \mid S, s \in C_i^*} \leq 8 c^4 \cdot \OPT_1 (C_i^*)\,.
    \end{align*}
    \label{lem:after_first_iteration}
\end{lemma}

\subsubsection{Dasgupta's potential argument}

Consider a run of \rejsam and let $S_i = \{s_1, s_2, \ldots, s_i\}$ denote the first $i$ centers chosen by \rejsam (for notational convenience, we let $S_0 = \emptyset$). 
We say that a cluster $C_j^*$ of the optimal solution is covered by $S_i$ if one of its centers is in $C_j^*$. 
Otherwise we say that this cluster is uncovered. 
For $i\in \{0,1, \ldots, k\}$, let $H_i$ and $U_i$ denote the set of all points from $P$ that, with respect to $S_i$, belong to covered and uncovered optimal clusters, respectively. 
Also let  $u_i$ denote the number of uncovered clusters after $i$ centers were opened. 
Finally, we say that a center $s_i$ is wasted if  $s_i \in H_{i-1}$, i.e., if the $i$:th center $s_i$ does not cover a previously uncovered cluster.

The following is an immediate corollary of the two preliminary lemmas; it is Corollary~$6$ in~\cite{noisykmeanspp}.
\begin{corollary}
    For any $i\in [k]$,
    \begin{align*}
        \ee{\Phi(H_i, S_i)} \leq 8c^4 \cdot \OPT_k(P)\,.
    \end{align*}
    \label[corollary]{cor:cover_bound}
\end{corollary}

The above corollary, says that the cost of covered clusters is at most a constant times the cost of an optimal solution. 
To bound the expected cost of uncovered clusters we use the argument of~\cite{Dasgupta}. 
It is based on a potential function argument. 
Define $W_i$ to be the number of wasted centers among the first $i$ centers. Hence $W_i$ equals $i$ minus the number of covered clusters. Further, let
\begin{align}
    \Psi_i =  W_i \cdot \frac{\PhiLSH(U_i, S_i)}{u_i}\,.
    \label{eq:potential}
\end{align}
Our potential $\Psi_i$ is different from the one used in~\cite{Dasgupta} in that we use $\PhiLSH$ instead of $\Phi$. This is the main difference and it is crucial for our analysis. 

For intuition, note that, for $i=0$, we have no wasted centers and all clusters are uncovered. So $W_0 = 0$ and $u_0 = k$ and $\Psi_0 = 0$. At the other end (for $i=k$), we have that the number of wasted centers equals the number of uncovered clusters, i.e., $W_k = u_k$, and so $\Psi_k$ equals the total cost of uncovered clusters. The definition of $\Psi_i$ allows us to bound this cost step-by-step. In particular, 
we will bound the expected increase of~\eqref{eq:potential} from $i$ to $i+1$, i.e., $\ee{\Psi_{i+1}- \Psi_i}$. We emphasize that the analysis is close to a verbatim transcript of that in~\cite{Dasgupta}; it is included for completeness.

In the following,  we let $\cF_{i}$ denote the realization of \rejsam of the first $i$ centers. Any realization $\cF_{i}$ determines e.g.\, the values of $\PhiLSH(U_{i}, S_{i})$ and $u_{i}$. 

We consider two cases: when the new center is in an uncovered cluster (\cref{lemma:uncovered_cluster}) and when it is in a covered cluster (\cref{lemma:covered_cluster}). 
\begin{lemma}[Lemma~8 in~\cite{Dasgupta}]
    Suppose that the $(i+1)$:th center $s= s_{i+1}$ is chosen in $U_i$. Then for any $\cF_i$
    \begin{align*}
        \ee{\Psi_{i+1} - \Psi_i\mid \cF_i, \{s\in U_i\} } \leq 0\,.
    \end{align*}
    \label[lemma]{lemma:uncovered_cluster}
\end{lemma}
\begin{proof}
    When $s$ belongs to an uncovered cluster $A$, we have $H_{i+1} = H_i \cup A$, $W_{i+1}= W_i$, $U_{i+1} = U_i \setminus A$, and $u_{i+1} = u_i -1$. Hence, using the monotonicity of the LSH data structure,
    \begin{align*}
        \Psi_{i+1} &= \frac{W_{i+1} \cdot \PhiLSH(U_{i+1}, S_{i+1})}{u_{i+1}} \\
        & \leq \frac{W_i\cdot (\PhiLSH(U_i, S_i) - \PhiLSH(A, S_i))}{u_i -1}\,.
    \end{align*}
    Let us bound the cost $\PhiLSH(A, S_i)$ for a randomly chosen uncovered cluster $A$. Here we use the notation $A(s)$ to denote the uncovered cluster so that $s\in A$. Since a point $s \in U_i$ is sampled proportional to $\PhiLSH(s, S_i)$
    \begin{align*}
        &\ee{\PhiLSH(A(s),S_i) \mid \cF_i, \{s\in U_i\}} \\
         &\qquad = \sum_{A} \frac{\PhiLSH(A, S_i)}{\PhiLSH(U_i, S_i)}\cdot \PhiLSH(A, S_i) \\
         & \qquad \geq \frac{\PhiLSH(U_i, S_i)}{u_i}\,,
    \end{align*}
    where the sum is over the $u_i$ uncovered clusters $A$ and the last inequality is by the Cauchy-Schwarz inequality.  
    Thus, $\ee{\Psi_{i+1} \mid \cF_i, \{s \in U_i\}}$ is at most
    \begin{align*}
    & \frac{W_i}{u_i -1} \left( \PhiLSH(U_i, S_i) - \ee{\PhiLSH(A(s), S_i) \mid \cF_i, \{s\in U_i\}}\right)\\
        & \qquad \leq \frac{W_i}{u_i -1} \left( \PhiLSH(U_i, S_i) - \frac{\PhiLSH(U_i, S_i)}{u_i}\right) = \Psi_i\,.
    \end{align*}
    
\end{proof}

\begin{lemma}[Lemma~9 in~\cite{Dasgupta}]
    Suppose that  $(i+1)$ center $s=s_{i+1}$ is chosen in $H_i$. Then for any $\cF_i$, $\Psi_{i+1}- \Psi_i \leq \PhiLSH(U_i, S_i)/ u_i$.
    \label[lemma]{lemma:covered_cluster}
\end{lemma}
\begin{proof}
    When $s$ is chosen from a covered cluster, we have $H_{i+1} = H_i, U_{i+1} = U_i, u_{i+1} = u_i$, and $W_{i+1} = W_i + 1$. Thus by the monotonicity of our data structure
    \begin{align*}
        \Psi_{i+1} - \Psi_i 
       & = \frac{W_{i+1}\cdot \PhiLSH(U_{i+1}, S_{i+1})}{u_{i+1}} - \frac{W_{i}\cdot \PhiLSH(U_{i}, S_{i})}{u_{i}} \\
        & \leq  \frac{(W_{i}+1)\cdot \PhiLSH(U_{i}, S_{i})}{u_{i}} - \frac{W_{i}\cdot \PhiLSH(U_{i}, S_{i})}{u_{i}} \\
        &= \frac{\PhiLSH(U_i, S_i)}{u_i}\,.
    \end{align*}
\end{proof}

Putting these two lemmas together gives a bound on the expected increase of the potential.
\begin{lemma}[Lemma~$10$ in~\cite{Dasgupta}]
    For $i\in \{0,1,\ldots, k-1\}$ and $\cF_i$, we have 
    \begin{align*}
        \ee{\Psi_{i+1}- \Psi_i \mid \cF_i} \leq \frac{\PhiLSH(H_i, S_i)}{ k-i}\,.
    \end{align*}
    \label[lemma]{lem:psi_bound}
\end{lemma}
\begin{proof}
    We have that $\ee{\Psi_{i+1} - \Psi_i \mid \cF_i}$ equals the sum of 
    \begin{align*}
         \ee{\Psi_{i+1} - \Psi_i\mid \cF_i, \{s\in U_i\} }\cdot \Pr[s\in U_i \mid \cF_i] 
    \end{align*}
    and
    \begin{align*}
        \ee{\Psi_{i+1} - \Psi_i\mid \cF_i, \{s\in H_i\} }\cdot \Pr[s\in H_i \mid \cF_i]\,. 
    \end{align*}
    Now, by \cref{lemma:uncovered_cluster} and \cref{lemma:covered_cluster} together with the fact that  \rejsam samples a center in $H_i$ with probability $\frac{\PhiLSH(H_i, S_i)}{\PhiLSH(P, S_i)}$, we can upper bound this sum by
    \begin{align*}
          0 + \frac{\PhiLSH(U_i, S_i)}{u_i} \cdot \frac{\PhiLSH(H_i, S_i)}{\PhiLSH(P, S_i)} 
        \leq \frac{\PhiLSH(H_i, S_i)}{u_i}  \leq \frac{\PhiLSH(H_i, S_i)}{k-i}\,.
    \end{align*}
\end{proof}

We are now ready to bound the overall cost of \rejsam.
\begin{theorem}[Theorem 11 in Dasgupta]
    If $S_k = S$ are the centers returned by \rejsam then 
    \begin{align*}
        \ee{\Phi(P, S)} \leq 8 c^6 (\ln(k) + 2) \OPT_k(P)\,.
    \end{align*}
\end{theorem}
\begin{proof}
    Using $\Phi(P, S) = \Phi(H_k, S) + \Phi(U_k, S)\leq \Phi(H_k, S) + \PhiLSH(U_k, S) = \Phi(H_k, S) + \Psi_k$, we have
    \begin{align*}
        \ee{\Phi(P, S)}& \leq \ee{\Phi(H_k, S)} + \sum_{i=0}^{k-1} \ee{\Psi_{i+1} - \Psi_i} \\
        &  \leq \ee{\Phi(H, S)} + \sum_{i=0}^{k-1} \ee{\frac{\PhiLSH(H_i, S_i)}{k-i} } \\
        & \leq \ee{\Phi(H, S)} + c^2 \sum_{i=0}^{k-1} \ee{\frac{\Phi(H_i, S_i)}{k-i} } \\
        & \leq 8c \cdot \OPT_k(P) + 8c^6 \sum_{i=0}^{k-1} \frac{\OPT_k(P)}{k-i}\\
        & \leq 8c^6 (\ln(k) + 2) \OPT_k(P)\,, 
    \end{align*}
    where the second inequality is by \cref{lem:psi_bound}, the third inequality is by the assumption that the LSH data structure is successful and thus returns $c$-approximate distances, and the penultimate inequality is by \cref{cor:cover_bound}.
\end{proof}

\subsection{Main Theorem for \rejsam Algorithm}
\rejsammain*
\begin{proof}
With probability at least $1-1/n$, the LSH data structure is successful and we will show that the statements of the theorem holds  if that is the case.
We start by showing that \rejsam samples points $x$ that are at most a factor $c^2$ away from the $D^2$-distribution. From \cref{lemma:dist} we know that the probability of sampling any point $x$ is
$
\frac{\dist(x, \text{Query}(x))^2}{\sum_{y \in P} \dist(y, \text{Query}(y))^2}\,. 
$
Since (by the assumption that the data structure is successful) we have that 
$
\dist(x,S) \leq \dist(x, \text{Query}(x)) \leq c \cdot \dist(x,S),
$ so we have
\begin{align*}
\frac{\dist(x,S)^2}{c^2\sum_{y \in P} \dist(y, S)^2} \leq &\frac{\dist(x, \text{Query}(x))^2}{\sum_{y \in P} \dist(y, \text{Query}(y))^2} \\
\leq & c^2 \cdot \frac{\dist(x,S)^2}{\sum_{y \in P} \dist(y, S)^2}.
\end{align*}
The time to initialize the multi-tree embedding (\mtinit) is $O(n d \log(d\Delta))$, the time to initialize the data structure used by \mtopen and \mtsample is $O(n \log(d\Delta))$, the total running time of $\mtopen$ is $O(n \log(d\Delta) \log n)$ (\cref{lemma:mtopen-fast-runtime}) and the running time of each call to \mtsample is $O(\log n)$ (\cref{lemma:mtsample-fast-runtime}).
 Finally, by \cref{lemma:numberiterations}, the expected number of iterations of the loop in \rejsam is ${O}(c^2 d^2 k)$, and the running time of each iteration is dominated by the running time of the Insert and Query operations, which is ${O}\left(d\log(\Delta)\cdot \left(n\log(\Delta )\right)^{O(1/c^2)}\right)$ by \cref{thm:data_structure}. Hence the total running time is $O\left(n \log(d \Delta)( d+  \log n) + k c^2 d^3 \log(\Delta) \cdot\left(n\log(\Delta )\right)^{O(1/c^2)}\right)$. 
The analysis of the approximation guarantee is presented in the previous section.
\end{proof}

\section{Variance of the Experiments and Aspect Ratio} \label{app:var}
\cref{var:song_cost} and \cref{var:kddcup_cost} presents the variance of the experiments. Recall that the numbers are reported over $5$ runs.

The assumption of bounded aspect ratio allows a clean presentation of the result. The dependency can, for example, be removed (using ideas from prior works) if we have a  very rough estimate of the optimum solution (e.g., within a factor $n$ or even $n^{10}$). Indeed, in that case,  we can obtain an instance in which  each coordinate of each point is an integer in range $[1, $poly$(n)]$ by losing a factor $1+1/n$ in the approximation guarantee (see~\cite{DBLP:conf/focs/AhmadianNSW17}). This bounds $\log \Delta = O(\log (nd))$. In practice this can be achieved very efficiently. In order to bound the height of the tree, we propose the following:
\begin{itemize}
\item We first compute an estimate of optimum solution by sampling a solution of $20$ randomly chosen points from the input. Then we compute the cost of this solution by assigning each point to the closest in the solution.
\item Then we divide this value by number of point and number of coordinate and $200$. This is intuitively the error that we let each coordinate make. The factor $200$ is chosen to ensure that the total error made is within $0.5\%$ of the considered optimum value. The call this value the scaling factor.
\item Afterwards, for each dimension of each point we divide it by the scaling factor and remove the fraction. For instance if the value of a considered coordinate is $1.2345$ and scaling factor is $0.01$, the resulting value would be $123$.
\end{itemize}
\begin{table*}[t]
  \centering
  \resizebox{\columnwidth}{!}{%
  \begin{tabular}{|c|c|c|c|c|c|c|}
    \hline
    Algorithm & $k=100$ & $k=500$ & $k=1000$ & $k=2000$ & $k=3000$ & $k=5000$ \\
    \hline
    \fastkm & 75364 & 169739 & 88843 & 92564 & 24225 & 40731 \\
    \hline  
    \rejsam & 288718 & 215658 & 74654 & 68922 & 87984 & 75364\\
    \hline  
    \kmpp   & 223686 & 64796 & 26784 & 20958 & 20881 & 30295\\
    \hline  
    \kmc    & 393782 & 121318 & 82700 & 22299 & 26945 & 15460\\
    \hline 
    \random & 687634 & 294580 & 147379 & 189350 & 182828 & 132779\\
    \hline  
  \end{tabular}
  }
  \smallskip
  
  \caption{The variance of the solutions of the algorithms for 
    the Song dataset for various values of $k$. All the numbers are scaled down by a factor $10^5$.}
    \label{var:song_cost}
\end{table*}

\begin{table*}[t]
  \centering
  \resizebox{\columnwidth}{!}{%
  \begin{tabular}{|c|c|c|c|c|c|c|}
    \hline
    Algorithm & $k=100$ & $k=500$ & $k=1000$ & $k=2000$ & $k=3000$ & $k=5000$ \\
    \hline
    \fastkm& 27110 & 672 &  813 & 86 & 77 & 163 \\
    \hline  
    \rejsam& 20440 & 1631 & 799 & 290 & 227 & 86 \\
    \hline  
    \kmpp & 8294 & 996 & 269 & 205 & 42 & 24\\
    \hline  
    \kmc  & 11529 & 830 & 883 & 204 & 495 & 135\\
    \hline  
    \random & 567214 & 290954 & 24118 & 23299 & 8770 & 23243\\
    \hline  
  \end{tabular}
  }
  \smallskip
  
  \caption{The variance of the solutions of the algorithm for 
    KDD-Cup dataset for various values of $k$. All the numbers are scaled down by a facto $10^2$.}
        \label{var:kddcup_cost}
\end{table*}

\fi

\end{document}